\documentclass[12pt]{elsarticle}
\usepackage{times}
\usepackage{xcolor}
\usepackage[utf8]{inputenc}
\usepackage{caption}
\usepackage{geometry}

\newcommand{\mar}{2.54cm}
\usepackage{pgfplots}
\usepgfplotslibrary{groupplots}
\usepackage{color}
\usepackage{amsmath,amssymb,amsthm}
\usepackage{url}
\theoremstyle{plain}
\newtheorem{theorem}{Theorem}
\newtheorem{lemma}{Lemma}

\newtheorem{proposition}{Proposition}
\newtheorem{coro}{Corollary}
\theoremstyle{definition}
\newtheorem{definition}{Definition}
\newtheorem{example}{Example}

\usepackage[ruled,vlined,linesnumbered]{algorithm2e}
\SetKw{Continue}{continue}
\usepackage{subcaption}
\usepackage{xspace}
\usepackage{adjustbox}
\usepackage{textcomp}
\usepackage{mathtools}
\usepackage{multirow}
\usepackage{cutwin}

\usepackage{tikz}
\usetikzlibrary{automata,arrows,shapes}
\usepackage{mathrsfs}
\pgfdeclarelayer{background}
\pgfsetlayers{main,background}
\tikzstyle{vertex}=[circle,fill=black!10,minimum size=20pt,inner sep=0pt]
\tikzstyle{edge} = [draw,thick]
\tikzstyle{weight} = [font=\small]
\tikzstyle{redundant edge} = [draw,line width=5pt,-,red,opacity=.7]

\usepackage{pgfplots}
\pgfplotsset{width=.30\linewidth,compat=1.9}

\DeclarePairedDelimiterX\Set[1]\{\}{%
  
  #1 }

\DeclareMathOperator*{\argmax}{arg\:max}

\newcommand{\norm}[1]{\lVert#1\rVert}

\newcommand{\algorithmfootnote}[2][\footnotesize]{%
  \let\old@algocf@finish\@algocf@finish
  \def\@algocf@finish{\old@algocf@finish
    \leavevmode\rlap{\begin{minipage}{\linewidth}
    #1#2
    \end{minipage}}%
  }%
}

\AtBeginDocument{}
\setkeys{Gin}{draft=false}      
\usepackage[obeyDraft,colorinlistoftodos]{todonotes}
\colorlet{SL}{red!20!yellow!60!white}
\colorlet{SLline}{SL!80!black}

\newcommand{\mahjong}{\mathcal{M}}
\newcommand{\M}{\mathcal{M}}
\newcommand{\dist}{{\sf{dist}}}
\newcommand{\dfncy}{{\sf{dfncy}}}
\newcommand{\sset}{{\sf{Set}}}
\newcommand{\Res}{{\sf{Res}}}
\newcommand{\cost}{{\sf{cost}}}
\newcommand{\discard}{{\sf{discard}}}
\newcommand{\val}{{\sf{val}}}

\usepackage{graphicx}
\usepackage{amssymb}
\usepackage{lineno}

\journal{ArXiv}

\begin{document}

\begin{frontmatter}


\title{Let's Play Mahjong!}


\author[uts]{Sanjiang Li}
\address[uts]{Centre for Quantum Software and Information,   University of Technology Sydney, Sydney, Australia}
\ead{Sanjiang.Li@uts.edu.au}

\author[snnu]{Xueqing Yan}
\address[snnu]{School of Computer Science, Shaanxi Normal University, Xi'an, China}
\ead{XueqingYan@snnu.edu.cn} 
\begin{abstract}
Mahjong is a very popular tile-based game commonly played by four players. Each player begins with a hand of 13 tiles and, in turn, players draw and discard (i.e., change) tiles until they complete a legal hand using a 14th tile. In this paper, we initiate a mathematical and AI study of the Mahjong game and try to answer two fundamental questions: how bad is a hand of 14 tiles? and which tile should I discard? We define and characterise the notion of deficiency and present an optimal policy to discard a tile in order to increase the chance of completing a legal hand within $k$ tile changes for each $k\geq 1$.  
\end{abstract}

\end{frontmatter}


\section{Introduction}
\label{S:1}
From the very beginning of AI research,  from checker \cite{Samuel59}, chess \cite{shannon1988programming}, Go \cite{Silver+16} to poker \cite{bowling2015heads} and StarCraft II\footnote{https://deepmind.com/blog/alphastar-mastering-real-time-strategy-game-starcraft-ii/}, games have played as test-beds of many AI techniques and ideas. In the past decades, we have seen AI programs that can beat best human players in perfect information games including checker, chess and Go, where players know everything occurred in the game before making a decision. Imperfect information games are 
more challenging. Very recently, important progress has been made in solving the two-player heads-up limit  Texas hold'em poker \cite{bowling2015heads} and its no-limit version \cite{BrownS17}, which are the smallest variants of poker played competitively by humans. In this paper, we initiate a mathematical and AI study of the more popular and more complicated Mahjong game.

Mahjong is a very popular tile-based multiplayer game played worldwide. The game is played with a set of 144 tiles based on Chinese characters and symbols (see Figure~\ref{fig:mjtiles})  and has many variations in tiles used, rules, and scores \cite{Wiki_mahjong}. Each player begins with a hand of 13 tiles and, in turn, players draw and discard (i.e., change) tiles until they complete a legal hand using a 14th tile.

In this paper, for simplicity, we only consider Mahjong-0, the very basic and essential version of Mahjong. The other variations can be dealt with analogously. There are only three types of tiles used in Mahjong-0:
\begin{itemize}
    \item Bamboos: $B1, B2, ..., B9$, each with four identical tiles
    \item Characters: $C1, C2, ..., C9$, each with four identical tiles
    \item Dots: $D1, D2, ..., D9$, each with four identical tiles
\end{itemize}
In this paper, we call \emph{Bamboo} ($B$), Character ($C$), Dot ($D$) \emph{colours} of the tiles, and write $\M_{0}$ for the set of tiles in Mahjong-0, which include in total 108 tiles. 

\begin{figure}[tbp]
    \centering
    \includegraphics[width=0.95\textwidth]{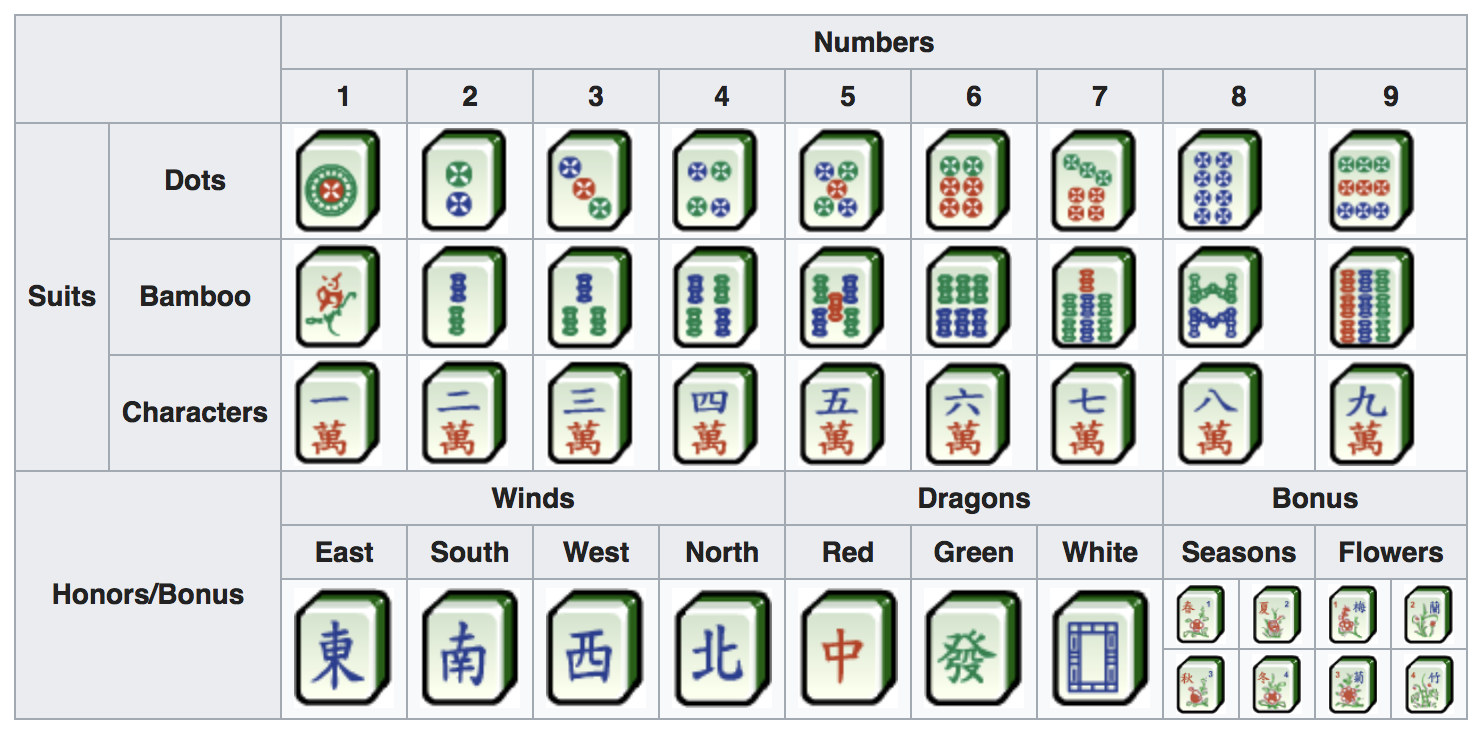}
    \caption{Mahjong tiles, from https://en.wikipedia.org/wiki/Mahjong}\label{fig:mjtiles}
    \label{fig:my_label}
\end{figure}

\begin{definition}
An \emph{eye} is a pair of identical tiles, a \emph{pong} (\emph{kong}) is a sequence of three (four) identical tiles. A \emph{chow} is a sequence of three consecutive tiles of the same colour. A \emph{meld} is either a pong or a chow. 
\end{definition}

In the analysis of this paper, we need the following non-standard notion.
\begin{definition}\label{dfn:pmeld}
A \emph{pseudochow} (\emph{pchow} for short) is a pair of tiles of the same colour such that it becomes a chow if we add an appropriate tile with the same colour. A \emph{pseudomeld} (\emph{pmeld} for short) is a pchow or a pair.  We say a tile $c$ \emph{completes} a meld $(ab)$ if $(abc)$ (after alphabetic sorting) is a meld. Similarly, we say a pair is completed from a single tile $t$ if it is obtained by adding an identical tile to $t$.
\end{definition}

For example, $(B3B4B5)$ is a chow, $(C1C1)$ is an eye, $(B7B7B7)$ is a pong, $(D9D9D9D9)$ is a kong, and $(B1B3)$ and $(C2C3)$ are two pchows.

The remainder of this paper is organised as follows. In Section~2, we introduce basic notions related to pure and hybrid 14-tiles, and characterise how bad a pure or hybrid 14-tile is in, respectively, Sections 3 and 4. Section~5 then presents an optimal policy for discarding tiles, in order to increase the possibility of completing a 14-tile within at most $k$ tile changes. Concluding remarks and directions for future research are given in Section~6. 

\section{Pure and hybrid 14-tiles}

Recall that each player begins with a set of 13 tiles (often called a \emph{hand}). Then players may change a tile by either drawing a new tile or robbing a tile, say $t$, discarded by another player if the player has two identical tiles $t$ and wants to consolidate a pong $(ttt)$ or if $t$ completes her 13-hand into a legal hand.\footnote{In $\mahjong_0$ we assume that kongs play no role in forming legal hands.} In this paper, we take the following defined 14-tiles as our object of study. 
\begin{definition}
\label{dfn:14-tile}
A 14-tile is a sequence $S$ of 14 tiles from $\M_0$. A sequence of tiles is \emph{pure} if they are all of the same colour and is \emph{hybrid} if otherwise.
\end{definition}
Apparently, no identical tiles can appear more than four times in a 14-tile. To be precise, we often write a tile $t$ as a pair $(c,n)$ such that $c \in \{0,1,2\}$ denotes the colour of $t$ and $n$ denotes the number of $t$. Here we assume that $0, 1,2$ denote, respectively, Bamboo, Character, and Dot. For example, $(0,3)$ and $(1,5)$ denote $B3$ and $C5$ respectively. Moreover, we assume that each 14-tile is represented in the standard form in the following sense.

\begin{definition}
\label{dfn:standard14-tile}
The \emph{standard form} of a (hybrid) 14-tile is a sequence $H$ of pairs $(c,n)$ with $0\leq c \leq 2$ and $1\leq n \leq 9$. We denote as $H[i]=(c_i,n_i)$ the $(i+1)$-th tile in $H$ and require, for each $0\leq i \leq 13$, that $c_i \leq c_{i+1}$ and $n_i \leq n_{i+1}$ if $c_i = c_{i+1}$.
We assume that there are no five identical tiles.
\end{definition}

For the purpose of this work, a pure 14-tile of Bamboos is not different from a pure 14-tile of Dots or a pure 14-tile of Characters. Thus we often write a pure 14-tile simply as a 14-tuple of integers from 1 to 9 if the colour of the tiles is clear from the context or not important. Formally, we have
\begin{definition}
A \emph{pure} 14-tile $S$ is a sequence of 14 integers with non-decreasing integer value from 1 to 9 and no value appears more than four times. 
\end{definition}
Suppose $S$ is a sequence of 14 integers. Write $S[i]$ for the $(i+1)$-th value of $S$ for $0 \leq i \leq 13$, i.e., we assume, as in Python and many other programming languages, that $S[0]$ represents the first value of $S$. Then $S$ is a pure 14-tile if and only if:
\begin{itemize}
    \item $S[i] \in \mathbb{N}$ and $1\leq S[i] \leq 9$ for each $0\leq i \leq 13$; and
    \item $S[i] \leq S[i+1]$ for any $0\leq i \leq 12$; and
    \item $S[i] < S[i+4]$ for any $0 \leq i \leq 9$.
\end{itemize}

For each pure 14-tile $S$, we define its \emph{set sequence}, written $\sset(S)$, as the subsequence of $S$ obtained by deleting, for each tile $(c,n)$ in $S$, all but the first occurrences of $(c,n)$ in $S$, and define its \emph{residual sequence} as  $\Res(S) \equiv S\setminus \sset(S)$.\footnote{In this paper, we regard a sequence as a set and identical items in the same sequence are considered as different elements in the corresponding set. Thus, suppose $W$ is a subsequence of $V$, $V\setminus W$ denotes the subsequence with items in $W$ removed from $V$.}
For example, consider the following pure 14-tile $S$ and its set and residual sequences:
\begin{align*}
    S \quad &=\ (1,1,1,2,2,3,3,4,7,7,7, 8,8,9,9) \\
    \sset(S) &=\ (1,2,3,4,7,8,9) \\
    \Res(S) & =\ (1,1,2,3,7,7,8,9)
\end{align*}
For convenience, we also write these sequences compactly as\\ \quad $S=(111223347778899)$, $\sset(S)=(1234789)$, and $\Res(S)=(11237789)$.

\begin{definition}
\label{dfn:complete&decomposition}
A 14-tile over $\M_0$ is \emph{legal} or \emph{complete} if it can be decomposed into four melds and one eye. Given a complete 14-tile $T$, a \emph{decomposition} $\pi$ of $T$ is a sequence of five subsequences of $V$ such that 
$\pi(5)$ is a pair (the eye) and, for $1\leq i\leq 4$, each $\pi(i)$ is a meld.
\end{definition}
\begin{definition}
\label{dfn:14tile_graph}
Two 14-tiles $S, T$ over $\M_0$  are \emph{neighbours} if $S[i] = T[i]$ holds for all but exactly one $i \in [0,13]$. The 14-tile graph is the pair $(\mathcal{T},E)$, where $\mathcal{T}$ is the collection of all 14-tiles and two 14-tiles $S,T$ are connected by an edge if $S,T$ are neighbours. The distance between two 14-tiles $S,T$ is the length of a shortest path from $S$ to $T$ in the 14-tile graph. 
\end{definition}
\begin{definition}
\label{dfn:dfncy}
The \emph{deficiency number} (or simply \emph{deficiency}) of a 14-tile $S$ is defined recursively:
\begin{itemize}
    \item A 14-tile has deficiency 0 if it is complete;
    \item A 14-tile has deficiency 1 if it is not complete but has a neighbour which is complete;
    \item In general, for $\ell\geq 0$, a 14-tile has deficiency $\ell+1$ if it has no deficiency smaller than or equal to $\ell$ and it has a neighbour which has deficiency $\ell$.
\end{itemize}
If the deficiency of $S$ is $\ell$, we write $\dfncy(S) = \ell$.
\end{definition}

For example, the following 14-tile is complete and, thus, has deficiency 0.
\begin{align}\label{mjsol1}
H &= (B1B2B2B3B3B4B7B7B7)(C1C1)(D4D5D6)
\end{align}
In this example, we have the following decomposition
\begin{align*}
    \pi &= (B1B2B3)(B2B3B4)(B7B7B7)(D4D5D6)(C1C1).
\end{align*}

A 14-tile $T$ is often not complete. The deficiency number of $T$ measures how bad $T$ is by counting the number of necessary tile changes to make it complete.

In the following two sections, we consider how to compute the deficiency number of a pure and, respectively, a hybrid 14-tile. 

\section{Deficiency of pure 14-tiles}
In this section we prove the following theorem.
\begin{theorem}\label{thm:pure}
Any pure 14-tile has deficiency less than or equal to 3.
\end{theorem}
To prove this theorem, we need to make several simple observations. First, there exists a pure 14-tile that has deficiency 3:
\begin{align}\label{14-tile_dfncy=3}
V & = (1 1 2 2 5 5 6 6 8 888 99)
\end{align}
Indeed, this pure 14-tile contains only one meld, viz. $(888)$. To get four melds, we need to replace some tile with another one (with the same colour). Since each tile can be used to construct at most one meld or one eye, we need at least 3 tile changes to obtain four melds. Suppose without loss of generality this pure 14-tile is in colour $B$. Then we may replace both $B2$ with $B4$, obtaining two chows $(B4B5B6)$, and replace one $B8$ with $B9$, obtaining a pong $(B9B9B9)$.  This example leads to the following fact.
\begin{lemma}\label{fact<=3}
There exists a pure 14-tile with deficiency 3.
\end{lemma}
By Definition~\ref{dfn:pmeld}, a pair of tiles say $(Bi,Bj)$ is a pmeld if $|j-i| \leq 2$. In what follows, we assume that a meld (kong) can also be counted as a pmeld (pong).
\begin{lemma}
\label{fact2}
Any four tiles of the same colour contain a pmeld. 
\end{lemma}
\begin{proof}
Suppose $V=(abcd)$ is a set of four tiles of the same colour and $1\leq a \leq b \leq c \leq d \leq 9$. If $a=b$ or $b=c$ or $c=d$, we have a pair and thus a pmeld. Now, suppose $1\leq a < b < c < d \leq 9$. It's easy to see that $b-a > 2$, $c-b > 2$, and $d-c > 2$ cannot happen in the same time. Thus $V$ has at least one pmeld. 
\end{proof}
The following notion of `disjoint' pmelds/melds is important.
\begin{definition}
\label{dfn:disjoint}
Suppose $V$ is a sorted tuple of $k$ ($3\leq k\leq 14$) tiles of the same colour. We say two sorted pmelds $(uv),(xy)$ contained in $V$ are \emph{disjoint}, if $(uv)$ and $(xy)$ are not identical pairs and there exist indexes $0\leq i<j<p<q<k$ such that $u=V[i], v=V[j], x=V[p]$ and $y=V[q]$.

Similarly, two sorted melds $(abc)$ and $(xyz)$ are \emph{disjoint} if there exist indexes $0\leq i_1<i_2<i_3<i_4<i_5<i_6<k$ such that $a=V[i_1], b=V[i_2], c=V[i_3], x=V[i_4], y=V[i_5]$, and $z=V[i_6]$. 
\end{definition}
In the same sense, we can also define when a meld and a pmeld are disjoint, and when $V$ contains three or more disjoint pmelds.

Consider the tuple $W=(1248)$ of four tiles of the same colour. It has two pmelds, viz. (12) and (24), but they are not disjoint in our sense. Furthermore, consider $V=(111158)$. It has two identical pairs $(11)$, which are in a sense competing with each other and cannot both be useful in forming a legal hand. That's why we don't consider them as disjoint pmelds.   

For six tiles of the same colour, we have the following result.
\begin{lemma}
\label{fact3}
Suppose $V$ is a tuple of six tiles of the same colour. If $V$ contains no kong, then $V$ contains at least two disjoint pmelds. If, in addition, $V$ contains six pairwise different tiles, then $V$ contains a chow and a disjoint pchow or contains three disjoint pchows. 
\end{lemma}
\begin{proof}
Suppose $V=(abcdef)$ has no kong. Here we don't assume that $V$ is sorted. First, we consider the case when $V$ contains at least one pair, say $a=b$. By Lemma~\ref{fact2}, there is a pmeld in $(cdef)$, which cannot be $(aa)$ as $V$ has no kong. Thus $V$ contains two disjoint pmelds if it has no kong.  
Second, suppose $V$ contains six different tiles and $a<b<c<d<e<f$. It is easy to see that, among $b-a$, $c-b$, $d-c$, $e-d$, and $f-e$, there exist at most one that is greater than 2. In fact, this is because
$8= 9-1 \geq f-a = (f-e) + (e-d) + (d-c) + (c-b) + (b-a)\geq 1 + 1 + 1 + 1 + 1 = 5$ and $8 < 3+3+1+1+1 = 9$. This shows that $V$ contains one chow and a pchow or contains three disjoint pchows in this case.
\end{proof}

We next show that if $V$ contains two disjoint melds, then it can be completed within two tile changes.
\begin{proposition}\label{prop0}
If $V$ contains two disjoint melds, then $\dfncy(V)\leq 2$.
\end{proposition}
\begin{proof}
The proof is not difficult but lengthy, see~\ref{sec:appendix1}.
\end{proof}

For any pure 14-tile $V$, recall that $\sset(V)$ is the set sequence of $V$, containing all different tiles in $V$, and $\Res(V)\equiv V\setminus\sset(V)$, containing all the rest tiles. 

\begin{proof}[Proof of Theorem~\ref{thm:pure}]
We prove Theorem~\ref{thm:pure} case by case.

Suppose $|\sset(V)|=8$ or $|\sset(V)|=9$. 
We have at least two chows in $\{(123),(456),(789)\}$. By Proposition~\ref{prop0}, we know $\dfncy(V)\leq 2$.

Suppose $|\sset(V)| = 4$ or $|\sset(V)| = 5$. Then  $V$ contains at least two different pongs in both cases. By Proposition~\ref{prop0}, we have $\dfncy(V) \leq 2$.

Suppose $|\sset(V)|=6$ or $|\sset(V)|=7$. 
If $V$ contains two or more disjoint melds, by Proposition~\ref{prop0}, we know $\dfncy(V)\leq 2$.

Suppose $|\sset(V)|=7$ and $V$ does not contain two disjoint melds. Let $\sset(V)$ $=(abcdefg)$ with $a<b<c<d<e<f<g$. 
It is easy to see that $\sset(V)$ contains one chow.  We note that $\Res(V)$ must contain a pair, this is because, otherwise, we have  $\Res(V)=\sset(V)$ and $V$ shall contain two chows. Thus $V$ consists of one kong or pong and at least four different pairs. Clearly, $V$ can be completed within three tile changes in this case. That is, $\dfncy(V)\leq 3$.

Suppose $|\sset(V)|=6$ and $V$ does not contain two disjoint melds. In particular, it has no two differnt pongs. Let $\sset(V)=(abcdef)$. This implies that $V$ consists of one kong and five different pairs. Clearly, $V$ can be completed within three tile changes, by completing any three pairs into pongs. This shows $\dfncy(V)\leq 3$.
\end{proof}


It is not difficult to design brute force searching algorithms for counting how many pure 14-tiles that are complete and how many pure 14-tiles that have deficiency 1, 2, 3. 
Our implementation results show that there are 
\begin{itemize}
    \item 118800 valid pure 14-tiles
    \item 13259 complete pure 14-tiles
    \item 91065 pure tiles with deficiency 1
    \item 14386 pure tiles with deficiency 2
    \item 90 pure tiles with deficiency 3 
\end{itemize}
One natural question is, do we have more direct methods to determine the deficiency number of a 14-tile by, e.g., examining its structure? This is surely possibly. For example,  we have the following characterisation for pure 14-tiles with deficiency 3.
\begin{proposition}
\label{lem:7}
For any pure 14-tile $V$, $V$ has deficiency 3 if and only if either
\begin{itemize}
    \item [\rm (i)] $V$ consists of a kong and 5 different pairs; or
    \item [\rm (ii)] $V$ consists of five different pairs, a pong $(p,p,p)$ and a single tile $s$, and there exists $x$ in $V$ such that $(psx)$ (after ordering) is a chow in $V$ and $V\setminus(psx)$ contains no chows.
\end{itemize}
\end{proposition}
Consider $V = (11222344558899)$ as an example. $V$ consists of a pong $(222)$, a single tile $(3)$ and five pairs $(11)$,  $(44)$, $(55)$, $(88)$, $(99)$. Thus $\dfncy(V)=3$.

\section{Deficiency of hybrid 14-tiles}
When a hybrid 14-tile $H$ is complete? How many tile changes are necessary to complete it if it is not? These questions can be answered by computing the deficiency number of $H$. There is one problem: the deficiency number of $H$ is defined in a recursive way and, to tell if $H$ is $k+1$, one need to know in principle all 14-tiles with deficiency $k$. It will be more convenient if we can determine, just like Proposition~\ref{lem:7},  the deficiency of $H$ by analysing its structure directly.

In this section, we present a method for determining if a 14-tile is complete, and show that the worst 14-tiles have deficiency 6 and give a characterisation for these 14-tiles, and then give a more direct method for measuring the `worseness' of a hybrid 14-tile $H$, and prove that it is exactly the deficiency of $H$.

\subsection{The best and the worst 14-tiles}

We first show how to decide if a hybrid 14-tile $H$ is complete. For convenience, we introduce the following notations.

\begin{definition}
\label{dfn:H_b}
Given a hybrid 14-tile $H$, we write $n_b(H)$, $n_c(H)$, and $n_d(H)$ for the numbers of Bamboos, Characters, and Dots in  $H$. We denote by $H_b$, $H_c$, and $H_d$ for, respectively, the subsequences of Bamboos, Characters, and Dots in $H$. 
A hybrid 14-tile $H$ is called a bcd-type 14-tile if $n_b\geq n_c\geq n_d$.
\end{definition}

For the hybrid 14-tile $H$ in Eq.~\eqref{mjsol1}, we have $H_b=(122334777)$, $H_c=(11)$, and $H_d=(456)$, where we omit the colour symbol $B,C,D$ in $H_b,H_c$ and $H_d$. Since $n_b(H)=9$, $n_c(H)=2$, $n_d(H)=3$, $H$ is not a bcd-type 14-tile. But it is easy to see that every hybrid 14-tile $T$ has a corresponding bcd-type 14-tile $S$ such that $S$ and $T$ are identical in essence.

\begin{proposition}
\label{prop:complete_14tile}
A bcd-type 14-tile $H$ is complete if and only if one of the following conditions is satisfied:
\begin{itemize}
    \item $n_b=14$, $n_c=n_d=0$ and $H$ is a complete pure 14-tile; 
    \item $n_b=12$, $n_c=2$, $n_d=0$, $H_b$ consists of four disjoint melds and $H_c$ consists of a pair;
    \item $n_b=11$, $n_c=3$, $n_d=0$, $H_b$ consists of three disjoint melds and a pair and $H_c$ consists of a meld;
    \item $n_b = 9$, $n_c = 5$, $n_d = 0$, $H_b$ consists of three disjoint melds and $H_c$ consists of a meld and a pair;
    \item $n_b=9$, $n_c=3$, $n_d=2$, $H_b$ consists of three disjoint melds, $H_c$ consists of a meld, and $H_d$ consists of a pair;
    \item $n_b=8$, $n_c=3$, $n_d=3$, $H_b$ consists of two disjoint melds and a pair and each of $H_c$ and $H_d$ consists of a meld;
    \item $n_b=8$, $n_c=6$, $n_d=0$, $H_b$ consists of two disjoint melds and a pair and $H_c$ consists of two disjoint melds; 
    \item $n_b=6$, $n_c=6$, $n_d=2$, each  of $H_b$ and $H_c$ consists of two disjoint melds and $H_d$ consists of a pair;
    \item $n_b=6$, $n_c=5$, $n_d=3$, $H_b$ consists of two disjoint melds, $H_c$ consists of a meld and a pair, and $H_d$ consists of a meld.
    \end{itemize}
\end{proposition}

Knowing how to decide if a hybrid 14-tile is complete, we next present one method for deciding the deficiency number of a hybrid 14-tile.

\begin{lemma}\label{lem:distance}
A 14-tile $H$ has deficiency $\leq k$ if and only if there exists a sequence of $s\leq k$ 14-tiles $H_1$, $H_2$, ..., $H_s$, such that $H_s$ is complete, and $H_i$ is a neighbour of $H_{i-1}$ for every $1\leq i\leq s$, where $H_0 = H$.
\end{lemma}
\begin{proof}
We prove this by using induction on $k$. When $k=0$, the result trivially holds, as a 14-tile $H$ has deficiency $\leq 0$ if and only if it is complete. In general, suppose the result holds for any $i\leq k$. We prove that it also holds for $i=k+1$. On one hand, suppose $H$ has deficiency $\leq k+1$. If $\dfncy(H) \leq k$, then by inductive hypothesis, there exists a sequence of length $\leq k$ that satisfies the desired condition. Now assume $\dfncy(H)=k+1$. By definition, there exists a neighbour $T$ of $H$ such that $\dfncy(T)=k$. By inductive hypothesis again, there exists a sequence $T_0=T$, $T_1$, ..., $T_s$ with $s\leq k$ such that  $T_s$ is complete and $T_i$ is a neighbour of $T_{i-1}$ for $1\leq i\leq s$. Apparently, $H$, $T_0$, $T_1$, ..., $T_s$ is a desired sequence of 14-tiles with length $\leq k+1$ for $H$. 

On the other hand, suppose there exists a desired sequence of $s\leq k+1$ 14-tiles $H_1$, $H_2$, ..., $H_{s}$ for $H=H_0$. If $s\leq k$, then by inductive hypothesis, $\dfncy(H)\leq k < k+1$. If $s=k+1$, then, by inductive hypothesis and that $H_2$, $H_3$, ..., $H_s$ is a desired sequence of length $k$ for $H_1$, we know $\dfncy(H_1)\leq k$. Since $H_1$ is a neighbour of $H$, we know by definition $\dfncy(H)\leq k+1$. 
\end{proof}
Since a tile change from a 14-tile $H$ leads to a neighbour of $H$, we have the following corollary. 
\begin{coro}\label{coro:1}
A 14-tile $H$ has deficiency $k$ if and only if it can be changed into a complete 14-tile by $k$ tile changes but cannot be changed into a complete 14-tile by $k-1$ or less tile changes.
\end{coro}
We next show that every 14-tile can be completed within six tile changes.
\begin{proposition}
\label{prop:dfncy<=6}
For any 14-tile $H$, the deficiency of $H$ is not greater than 6.
\end{proposition}
\begin{proof}
Without loss of generality, we assume that $H$ is a bcd-type 14-tile, i.e., $n_b(H)\geq n_c(H) \geq n_d(H)$. Clearly, $n_b(H)\geq 5$.

To show $\dfncy(H)\leq 6$, by Corollary~\ref{coro:1}, we show that $H$ can be completed within six tile changes. To this end, we need only show that $H$ contains two disjoint melds, or a meld and a disjoint pmeld, or three disjoint pmelds such that all involved pmelds are completable.

Suppose $H_b$ has at least eight tiles, i.e., $n_b(H)\geq 8$. If $H_b$ contains a pong $(aaa)$, then by Lemma~\ref{fact2} the remaining five or more Bamboo tiles contain at least one pmeld which does not involve $a$. In this case, we have a pong and a disjoint pmeld. If $H_b$ contains no pongs but has two different pairs, by Lemma~\ref{fact2} again, the remaining four or more Bamboo tiles contain a pmeld. If $H_b$ contains no pongs and has only one pair, then by Lemma~\ref{fact3} the remaining six or more tiles contain three disjoint pmelds. In case $H_b$ contains no pairs, it is easy to see that any set of eight different bamboo tiles contains two disjoint chows.  

If $6\leq n_b(H) \leq 7$, then by $n_b(H) \geq n_c(H) \geq n_d(H)$, we know $n_c(H) \geq 4$ and thus, by Lemma~\ref{fact2}, $H_c$ contains a pmeld. By Lemma~\ref{fact3}, $H_b$ contains either a kong (hence a pong) or two disjoint pmelds. So in this case, $H$ contains a pong and a disjoint pmeld or three disjoint melds.

If $n_b(H)=5$, then we have $n_c(H)=5$ and $n_d(H)=4$. By Lemma~\ref{fact2}, $H$ contains three disjoint pmelds.

Note that a pair $(xx)$ cannot be completed into a pong if and only if two identical $x$ have been used in the other melds or pmelds; and a pchow $(xy)$ cannot be completed into a chow $(xyz)$ if and only if all four identical $z$ have been used in the other melds or pmelds. Apparently, this is not the case for any pmeld involved here. Therefore, $H$ can be completed within six tile changes.
\end{proof}
For example, the following 14-tile has deficiency 6 as it only has three disjoint pmelds, e.g., $(B1B2)$, $(C2C2)$, and $(D6D8)$.
\begin{align}\label{eq:14-tile_dfncy6}
H &= (B1B1B2B5B8)(C1C2C2C5C8)(D3D6D8D9).
\end{align}

As proved in Theorem~\ref{thm:pure}, if $H$ is pure, then $\dfncy(H) \leq 3$. Similarly, if at least nine tiles are of the same colour, then the deficiency of $H$ is smaller than 6.  
\begin{proposition}\label{prop:dfncy<=5}
Given a 14-tile $H$, if $H_b$ contains 9 tiles, then $\dfncy(H)\leq 5$.
\end{proposition}
\begin{proof}
We show that $H_b$ contains two disjoint melds, or one meld and two disjoint pmelds, or four disjoint pmelds. We discuss this case by case. Note that if $H_b$ contains two pongs, or two disjoint melds, or four disjoint pairs, we are done. 

Suppose $H_b$ contains only one pong $(aaa)$. By Lemma~\ref{fact3} the remaining 6 tiles contains two pmelds $(bc)$ and $(de)$. Assume $b<c$ and $d<e$. If $a$ completes both $(bc)$ and $(de)$, then we have two chows. If $a$ completes only one of $(bc)$ and $(de)$, we have a pair $(aa)$, a chow and a pmeld. If $a$ completes neither $(bc)$ nor $(de)$, then we can complete $(bc)$ and $(de)$ in two tile changes without violating the constraint that there are no five identical tiles in $\mahjong_0$. 

Suppose $H_b$ contains no pongs.  If $H_b$ contains only three different pairs, then $H_b$ has the form $(aa)(bb)(cc)$ $(d)(e)(f)$ with $a<b<c$ and $d<e<f$. 
        \begin{itemize}
            \item Suppose $e-d \leq 2$ or $f-e \leq 2$. Then we have three pairs and one pmeld. 
            \item Suppose $e-d>2$ and $f-e>2$. If $(abc)$ is a meld, then we already have at least two melds. If $(abc)$ is not a meld and $(xyz)$ is a meld, where $x,y \in \{a,b,c\}$ and $z \in \{d,e,f\}$, then we have one meld $(xyz)$, one pmeld $(xy)$ and one pair. If there is no meld in $H_b$, since $H_b$ has six different tiles, we can show that there exist $x\in \{a,b,c\}$ and $y,z\in \{d,e,f\}$ such that $(yx)$ and $(xz)$ are two pmelds. Together with two other pairs,  we also have four disjoint pmelds in $H_b$.    
        \end{itemize}

Suppose $H_b$ contains no pongs and has only two different pairs. Then $\sset(H_b)$ has seven different tiles and thus at least one chow $(def)$ ($d<e<f$). Write $(aa)$ and $(bb)$ for the two pairs in $H_b$, where $a<b$. If $a,b$ are not in $\{d,e,f\}$, then we have one meld and two different pairs. If $a,b$ are both in $\{d,e,f\}$, then $(ab)$ is a pchow. Since there is a pmeld in the other four tiles of $H_b$, we also have one chow and two pmelds in this case. If, for example, $a\in \{d,e,f\}$ but $b\not\in\{d,e,f\}$, then we have one chow $(def)$, one pair $(bb)$, and another pmeld in the remaining four tiles of $H_b$. 

Suppose $H_b$ has only one pair. Then, since $\sset(H_b)$ has 8 different tiles, there are two disjoint chows in $H_b$.

This shows that $\dfncy(H) \leq 5$ if $H_b$ contains 9 tiles. 
\end{proof}
Next, we give a characterisation of the worst 14-tiles. To this end, we need the following notion of worst pure $k$-tiles for $1\leq k \leq 8$: 
\begin{definition}
Let $k$ be an integer between 1 and 8. A \emph{pure $k$-tile} $V$ is a sequence  of integers in $[0,9]$ such that no integer appears more than four times and $V[i] \leq V[i+1]$ for any $0\leq i<k-1$.  

We say a pure $k$-tile $V$ is a worst $k$-tile if 
\begin{itemize}
    \item when $1\leq k\leq 3$, $V$ contains no pmelds; 
    \item when $k=4$ or $k=5$, $V$ contains no meld and contains only one pmeld;
    \item when $k=6$, $V$ contains at most one meld or up to two disjoint pmelds;
    \item when $k=7$, $V$ contains up to two disjoint pmelds;
    \item when $k=8$, $V$ contains at most one meld and a disjoint pmeld, or up to three disjoint pmelds.
\end{itemize}
\end{definition}

\begin{proposition}\label{prop:dfncy=6}
For any 14-tile $H$, $\dfncy(H)=6$ if and only if one of the following situation occurs:
\begin{itemize}
    \item [1.] 
    $n_b(H) = 8$, $n_d(H) = n_c(H) = 3$, $H_c$ and $H_d$ contain no pmeld and $H_b$ is a worst pure 8-tile.
    
    \item [2.] $n_b(H) = 7$,  $n_c(H) =5$, $n_d(H)=2$, $H_b$ is a worst pure 7-tile and $H_c$ a worst pure 5-tile and $H_d$ has no pmeld.
    \item [3.] $n_b(H) = 6$,  $n_c(H) =5$, $n_d(H)=3$, $H_b$ is a worst pure 6-tile and $H_c$ a worst pure 5-tile, and $H_d$ has no pmeld.
   \item [4.] $n_b(H) = n_c(H) =5$, $n_d(H)=4$, $H_b$ and $H_c$ are worst pure 5-tiles and $H_d$ a worst pure 4-tile. 
\end{itemize} 
\end{proposition}

\subsection{A direct method for measuring the worseness of a 14-tile}

While in principle we can determine if the deficiency of a hybrid 14-tile is not large than $k$  by recursively checking if it has a neighbour with deficiency smaller than $k$, this is far from efficient. In this subsection, we provide a direct method for measuring the worseness of a 14-tile.

The notion of decomposition may also be extended to  incomplete 14-tiles.
\begin{definition}
Given any 14-tile $T$, a \emph{pseudo-decomposition} (\emph{p-decomposition} for short) of $T$ is a sequence $\pi$ of five subsequences, $\pi(1),...,\pi(5)$, of $T$ such that 
\begin{itemize}
    \item $\pi(5)$ is a pair, a single tile, or empty;
    \item for $1\leq i\leq 4$, each $\pi(i)$ is a meld, a pmeld, a single tile, or empty.
\end{itemize}
We write $\pi(0)\equiv T \setminus \bigcup_{i=1}^5 \pi(i)$  for the sequence of remaining tiles of $T$.
\end{definition}

\begin{definition}
\label{dfn:normal_dcmp}
A p-decomposition $\pi$ of a 14-tile $T$ is \emph{completable} if there exists a decomposition $\pi^*$ of a complete 14-tile $S$  such that each $\pi(i)$ is a subsequence of $\pi^*(i)$. If this is the case, we call $S$ a \emph{completion} of $T$ based upon $\pi$. 

The \emph{cost} of a completable p-decomposition $\pi$, written $\cost_T(\pi)$,  is the number of necessary tile changes to complete $\pi(5)$ into a pair and complete each $\pi(i)$ into a meld for $1\leq i\leq 4$.

A completable p-decomposition $\pi$ is \emph{saturated} if no tiles in $\pi(0)$ can be added into any $\pi(i)$ to form another completable p-decomposition.  
A p-decomposition $\pi$ of $T$ is called \emph{normal} if $\pi$ is saturated and has a completion $S$ which is a closest complete 14-tile of $T$.
\end{definition}

There may exist many different p-decompositions of $T$. If a p-decomposition $\pi$ is incompletable, we write $\cost_T(\pi)=\infty$. In the following, we show that $\dfncy(T)$ is identical to the minimal cost among all p-decompositions of $T$.

The following results will be useful, where we say a p-decomposition $\pi_1$ \emph{refines} another p-decomposition $\pi_2$ if $\pi_1(i)\supseteq \pi_2(i)$ for any $1\leq i\leq 5$.
\begin{lemma}\label{lem:saturated}
Suppose $\pi$ is a p-decomposition of a 14-tile $T$.
If $\pi$ is completable but not saturated, then there exists a saturated completable p-decomposition $\pi^*$ such that $\pi^*$ refines $\pi$ and $\cost_T(\pi^*) < \cost_T(\pi)$.
\end{lemma}
\begin{proof}
Suppose $\pi$ itself is not saturated. By definition, there exist $t\in \pi(0)$ and $1\leq i\leq 5$ such that $\pi(i)$ 
is not a meld (a pair) when $i\leq 4$ ($i=5$) and  adding $t$ to $\pi(i)$ results a new completable p-decomposition $\pi^*$. Clearly, $\cost_T(\pi^*)=\cost_T(\pi)-1$. If $\pi^*$ is not saturated, then we can further refine it till it is saturated. The conclusion then follows immediately. 
\end{proof}

The above lemma shows that we need only consider saturated completable p-decompositions. The following result states that the cost of a p-decomposition of $T$ is identical to the distance between $T$ and any completion of $T$ based upon $\pi$. 

\begin{lemma}
\label{lem:dcmp-dist}
Suppose $\pi$ is a saturated completable p-decomposition of a 14-tile $T$. Then, for any completion $S$ of $T$ based upon $\pi$, we have $\dist(S,T)=\cost_T(\pi)$.
\end{lemma}
\begin{proof}
Note that any completion $S$ is obtained by filling the `holes' in $\pi$. Since $\pi$ is saturated, these holes cannot be filled by using tiles in $\pi(0)$. Thus, each hole must be filled by a tile change. Therefore, $S$ has distance $\cost_T(\pi)$ with $T$.
\end{proof}

The following result shows that, very often, a saturated p-decomposition has no empty subsequences.

\begin{proposition}\label{prop:normal}
Suppose $T$ is a 14-tile and $\pi$ is a completable p-decomposition which has an empty subsequence. If $\pi(0)$ contains 4 or more tiles or $\pi(5)\not=\varnothing$, then there exists another completable p-decomposition $\pi^*$ which has no empty subsequences and $\cost_T(\pi^*) < \cost_T(\pi)$. 
\end{proposition}
\begin{proof}
 See \ref{sec:appendix2}.
\end{proof}
As a corollary, we have 
\begin{coro}\label{coro:emptyss}
Suppose $T$ is a 14-tile and $\pi$ a completable p-decomposition with $n\geq 2$ empty subsequences. There exists a completable p-decomposition $\pi^*$ s.t. $\cost_T(\pi^*) \leq\cost_T(\pi)-n$.
\end{coro}
\begin{proof}
If $\pi(5)=\varnothing$, then $\pi(0)$ contains at least $3n-1\geq 5$ tiles. We take the case when $n=2$ as an example. The other cases are analogous. By the proof of Proposition~\ref{prop:normal}, there exists a tile $t\in\pi(0)$ such that $T$ contains at most three identical tiles $t$ and thus we could refine $\pi$ by letting $\pi(5)=(tt)$ if $(tt)$ is contained in $T$ and letting $\pi(5)=(t)$ otherwise. If the first case holds, we already have the desired p-decomposition. In the second case, we have four tiles left in $\pi(0)$ and can apply Proposition~\ref{prop:normal} to obtain a new completable p-decomposition with a smaller cost. This also gives us the desired p-decomposition. 
\end{proof}

\begin{example}
\label{ex:T9}
Consider the 14-tile 
\begin{align}\label{eq:T9}
T &=(B1B1B2B2B2B2B3B3)(C1C2C8)(D2D2D8)
\end{align}
and its p-decompositions 
\begin{align}\label{eq:D(T9)0}
\pi_0 &= (B1B2B3)(B2B2B2)(B1B3)(D2D2)(D8) \\
\label{eq:D(T9)1}
\pi_1 &= (B1B2B3)(B1B2B3)(B3)(D2D2)(D8)\\
\label{eq:D(T9)2}
\pi_2 &= (B1B2B3)(B1B2B3)(C1C2)(D2D2)(B2B2).
\end{align}
The pmeld $(B1B3)$ in $\pi_0$ cannot be completed, as there are already four $B2$ in $\pi_0$. Both $\pi_1$ and $\pi_2$ are saturated and completable. For example, $\pi_1$ can be completed into $$(B1B2B3)(B1B2B3)(B3\underline{B4}\underline{B5})(D2D2\underline{D2})(D8\underline{D8})$$ and $\pi_2$ can 
be completed into $$(B1B2B3)(B1B2B3)(C1C2\underline{C3})(D2D2\underline{D2})(B2B2).$$ 
We have $\cost_T(\pi_1)=4$ and $\cost_T(\pi_2)=2$. Since $\dfncy(T)=2$, $\pi_2$ is a p-decomposition with the minimal cost.  
\end{example}

\begin{theorem}\label{thm:m-cost}
For any 14-tile $T$, the minimum cost among all p-decompositions of $T$ is the deficiency of $T$, i.e., \begin{align}
    \dfncy(T) = \min\{\cost_T(\pi) \mid \mbox{$\pi$ is a p-decomposition of $T$}\} .
\end{align} 
\end{theorem}
\begin{proof}
Let $k=\dfncy(T)$. We show there exists a p-decomposition $\pi$ of $T$ with $\cost_T(\pi)=k$. By Corollary~\ref{coro:1}, the distance of any complete 14-tile to $T$ is not less than $k$ and there exists a complete 14-tile $S$ that has  distance $k$ with $T$. Let $\pi^*$ be a decomposition of $S$ and let $\pi$ be the restriction of $\pi^*$ to $T$, that is, all tiles not in $T$ are removed from $\pi^*$. There are exactly $k$ tiles removed from $\pi^*$ and thus, by definition, we have $k=\cost_T(\pi)$. 

On the  other hand, for any p-decomposition $\pi$, if it is incompletable, then $\cost_T(\pi)=\infty$; and if it is completable, then by Lemma~\ref{lem:saturated} there exists a saturated p-decomposition $\pi'$ of $T$ such that $\cost_T(\pi')\leq \cost_T(\pi)$.   
Assume $\pi$ is a saturated p-decomposition $\pi$ of $T$.  By Lemma~\ref{lem:dcmp-dist}, there exists a completion $S$ of $T$ based upon $\pi$ such that 
 the distance between $T$ and $S$ is the same as $\cost_T(\pi)$. This shows that $k\leq \cost_T(\pi)$ for any saturated p-decomposition $\pi$ of $T$. 

Therefore, $k$ is the minimum cost among all p-decompositions of $T$.
\end{proof}
In what follows, we say a p-decomposition $\pi$ of a 14-tile $T$ is \emph{minimal} if $\cost_T(\pi) = \dfncy(T)$. We next present a method for determining the deficiency of a 14-tile by trying to construct a minimal p-decomposition of $T$.

 \paragraph{Our procedure}
 We construct a quadtree, where each node is denoted by a word $\alpha$ in the alphabet $\Sigma=\{1,2,3,4\}$ and we attach to each $\alpha$ a subsequence $S_\alpha$ of $T$, which denotes the set of tiles remaining to be processed, and a p-decomposition $\pi_\alpha$ of $T$, and maintain a queue $Q$ of nodes to be explored and a global current best value $\val$. Initially, the root node, denoted as the empty word $\varepsilon$, is put in $Q$. We define $S_\varepsilon=T$, the p-decomposition $\pi_\varepsilon$ with $\pi_\varepsilon(i)=\varnothing$ for $0\leq i\leq 5$, and let $\val=6$, the largest possible deficiency number.
 
 Suppose a node $\alpha$ is popped out from the queue $Q$ and suppose $\val$ is the current best value and $S_\alpha$ and $\pi_\alpha$ are,  respectively, the  subsequence and the p-decomposition associated with $\alpha$. We now \emph{expand} $\alpha$ as follows. Let $a=S_\alpha[0]$ be the first tile in the subsequence $S_\alpha$. We add  up to four child nodes $\alpha1$ to $\alpha4$ under the node $\alpha$. For each $\ell\in\{1,2,3,4\}$, the subsequence $S_{\alpha\ell}$ and the p-decomposition $\pi_{\alpha\ell}$ are obtained by \emph{reducing} $S_\alpha$ and \emph{refining} $\pi_\alpha$ respectively. 
 
 \paragraph{Notation} Given a tile $t=(c,n)$, we write $t^+$ ($t^{++}$, resp.) for $(c,n+1)$ ($(c,n+2)$, resp.) as long as $1\leq n\leq 8$ ($1\leq n\leq 7$, resp.); and write $ t^-$ ( $t^{--}$, resp.) for $(c,n-1)$ ($(c,n-2)$, resp.) as long as $2\leq n \leq 9$ ($3\leq n\leq 9$, resp.).
 
 \begin{itemize}
     \item Define $S_{\alpha1}=S_\alpha\setminus (a)$ and $\pi_{\alpha1}=\pi_\alpha$.
     
     \item If $a^+$ or $a^{++}$ is in $S_\alpha$,\footnote{If this is not the case, then this node is not introduced. Similar assumption applies to the other two child nodes.}
     define $S_{\alpha2}=S_\alpha \setminus (aa^+a^{++})$.\footnote{Note that if, for example, $(aa^+a^{++}) \cap S_\alpha=(aa^{++})$, then $S_\alpha \setminus (aa^+a^{++})=S_\alpha \setminus (aa^{++})$.} Suppose $i$  is the first index in $\{1,2,3,4\}$ such that  $\pi_\alpha(i)=\varnothing$. Define $\pi_{\alpha2}(i)=(aa^+a^{++})\cap S_\alpha$  and $\pi_{\alpha2}(j)= \pi_\alpha(j)$ for $j\not=i$.\footnote{Note that if $(a^-aa^+)$ is contained in $S$, then this chow shall have been put in some $\pi$ in a previous step when $a^-$ is examined.} 
    
    \item If $(aa)\subseteq S_\alpha$ and $\pi_\alpha(5)$ is empty, define $S_{\alpha3} = S_\alpha \setminus (aa)$ and $\pi_{\alpha3}(5)=(aa)$ and $\pi_{\alpha3}(j)=\pi_\alpha(j)$ for $1\leq j\leq 4$.

    \item If $(aa)\subseteq S_\alpha$,  define  $S_{\alpha4} = S_\alpha \setminus (aaa)$.
    Suppose $i$ is the first index in $\{1,2,3,4\}$ such that $\pi_\alpha(i)=\varnothing$. Define $\pi_{\alpha4}(i)=(aaa)\cap S_\alpha$  and $\pi_{\alpha4}(j)= \pi_\alpha(j)$ for $j\not=i$.
 \end{itemize}
 
For each child node $\alpha\ell$, if $\pi_{\alpha\ell}(i)\not=\varnothing$ for all  $1\leq i\leq 5$,
then we compare the value $\cost_T(\pi_{\alpha\ell})$ with the current best value $\val$, update $\val$ as $\cost_T(\pi_{\alpha\ell})$ if the latter is smaller, and terminate this branch. 

Suppose $\pi_{\alpha\ell}(i)$ is empty for some $i$ but there are no more tiles left to process, i.e., $S_{\alpha\ell}=\varnothing$. We may recycle from $\pi_{\alpha\ell}(0)$ some discarded tiles and put it in some subsequence of $\pi_{\alpha\ell}$ to get a refined completable p-decomposition.  Suppose $\pi_{\alpha\ell}(5)\not=\varnothing$ or  $\pi_{\alpha\ell}(0)$ contains four or more tiles. Let $n$ be the number of empty $\pi_{\alpha\ell}(i)$ for $1\leq i\leq 5$. Then, by Corollary~\ref{coro:emptyss}, there exists another completable p-decomposition $\pi'$ such that  $\cost_{T}(\pi') \leq \cost_T(\pi_{\alpha\ell})-n$. We then directly compare $\cost_T(\pi_{\alpha\ell})-n$ with $\val$ and update $\val$ as  $\cost_T(\pi_{\alpha\ell})-n$ if the latter is smaller, and terminate this branch. 

Suppose $S_{\alpha\ell}=\varnothing$, $\pi_{\alpha\ell}(5)=\varnothing$ and $\pi_{\alpha\ell}(0)$ contains two (three) different tiles $u,v$ ($u,v,w$). We set $\cost=\cost_T(\pi_{\alpha\ell})$ if $T$ contains two kongs $(uuuu)$ and $(vvvv)$ (three kongs $(uuuu)$, $(vvvv)$, and $(wwww)$), and set $\cost=\cost_T(\pi_{\alpha\ell})-1$ otherwise. Then we compare $\cost$ with $\val$ and update $\val$ as  $\cost$ if the latter is smaller, and terminate this branch. 

If $S_{\alpha\ell}\not=\varnothing$ and $\pi_{\alpha\ell}$ has an empty subsequence, then we put $\alpha\ell$ in $Q$. 

After $\alpha$ is expanded, if $Q$ is nonempty and $\val>0$, then we pop out another node $\beta$ from $Q$ and expand $\beta$ as above. The whole procedure is stopped either when we have $\val=0$ or when $Q$ is empty.

%
\section{Decision making}
Given a 14-tile $H$, we have developed methods for deciding if $H$ is complete, and in case it is not complete, measuring how good or bad it is. The next important task is to decide, in case $H$ is not complete, which tile should the agent discard?

We suppose the agent maintains a  \emph{knowledge base} $\omega$, which contains all her information about the available tiles. For simplicity, we represent $\omega$ as a 27-tuple, where $\omega[9c+n]$ ($0\leq c\leq 2$ and $1\leq n \leq 9$) denotes the number of identical tiles $t=(c,n)$ the agent \emph{believes} to be available. Initially, we have $\omega[9c+n] = 4$ for each tile $t=(c,n)$. When all players have their hands, the agent also has her hand $H$ and updates her $\omega$ accordingly as
\begin{align}
\omega[9c+n] = 4- \mbox{the number of  $(c,n)$ in $H$}. 
\end{align}
Then she continues to modify $\omega$ according to the process of the game. For example, if one player discards a tile $t=(c,n)$ and no pong is formed from the discard of $t$, then the agent updates its $\omega$ by decreasing by one its $\omega[9c+n]$ and leaves the other items unchanged. More advanced techniques will be investigated in future work, where we may record the history of every player, and \emph{infer} for each tile $t=(c,n)$, what is the most likely value of $\omega[9c+n]$. 

Now suppose the agent has a 14-tile $T$ and her current knowledge base is $\omega$. For simplicity, we write $\omega[9c+n]$ as $\omega(c,n)$ for any tile $t=(c,n)$ and 
write 
\begin{align}
\norm{\omega} \equiv \sum\{\omega(c,n) \mid 0\leq c\leq 2, 1\leq n\leq 9\}
\end{align}
for the number of available tiles.
Our task is to decide which tile she should discard. To this end, for each  $0\leq i \leq 13$, we associate $i$ with a value $\delta_{T,\omega}(i)$ which is defined as 
\begin{align}\label{eq:delta}
    \delta_{T,\omega}(i) &= \sum_{0\leq c\leq 2, 1\leq n\leq 9}\Big\{\omega(c,n) \mid \dfncy(T[i/(c,n)]) < \dfncy(T)\Big\}
\end{align}
Apparently, we change $T[i]$ to an available tile $(c,n)$ only if it is profitable, i.e., if the 14-tile obtained by replacing $T[i]$ with $(c,n)$, written $T[i/(c,n)]$, has a smaller deficiency number. For each index $i$, $\delta_{T,\omega}(i)$ denotes the number of available tiles $(c,n)$ (identical tiles are counted differently) such that $T[i/(c,n)]$ has a smaller deficiency than $T$.

After $\delta_{T,\omega}(i)$ is computed for each $0\leq i\leq 13$, we then could pick one index $i$ that has the maximum value
and discard $T[i]$, i.e.,
\begin{align}
\label{eq:discard}
    \discard(T,\omega) = \argmax_{T[i]\;:\; 0\leq i\leq 13} \delta_{T,\omega}(i).
\end{align}
\begin{example}\label{ex:decision}
For example, consider 
\begin{align}
H &= (B1B1B1B8B8B9)(C1C5C5C5)(D1D5D6D7)\\
\omega &= (111111111)(111111111)(111111111)
\end{align}
Here, for convenience, we assume that for each tile $t=(c,n)$ there is only one identical tile available, i.e., $\omega[i] = 1$ for any $0\leq i \leq 26$. 
Then we have $\dfncy(H)=2$ and $\delta_{H,\omega}=[0,0,0,3,3,7,6,0,0,0,6,0,0,0]$. In particular, we have $\delta_H(5)=7$, $\delta_H(6)=6$, and $\delta_H(10)=6$. Thus, according to Eq.~\eqref{eq:discard}, we should discard $H[5]=B9$. In this case,  changing $B9$ as any of $B8,C1,C2,C3$, $D1,D2,D3$ will decrease the deficiency number from $2$ to $1$.
\end{example}

Eq.~\eqref{eq:discard} provides a very good, and easy to compute, heuristics for deciding which tile the agent should discard when her goal is to win as early as possible. It is, however, not an optimal policy. 

\begin{example}\label{ex:optimal}
Consider the following example. 
\begin{align}
    H &= (B1B2B3B7B8B9)(C2C2C2)(D1D2D3D5D9)\\
    \omega &=
    (000000000)(000000000)(010110001)
\end{align}
Suppose our goal is to decide, given the knowledge base $\omega$, which tile in $H$ 
to discard if we want to increase our chance to \emph{win within 2 tile changes}. If we adopt the heuristics given in Eq.~\ref{eq:discard}, then by $\dfncy(H)=1$ and $\delta_{H,\omega}(12) = \delta_{H,\omega}(13) = 1$ and $\delta_{H,\omega}(i) = 0$ for $0\leq i < 12$, we should discard either $T[12]=D5$ or $T[13]=D9$. But we will see this is not exact, and discarding $D9$ is more profitable. To this end, we need to compute, for each $0\leq i\leq 13$, the chance of completing $H$ if we discard $H[i]$. Note that $D2,D4,D5$ and $D9$ are the only  tiles available. 

In the first tile change, if we discard any Bamboo tile, then, as there are no Bamboo tiles available, the corresponding chow will never be completed again. To make the revised 14-tile complete, we need at least three more tile changes. That is, the success chance is 0. This is also true if we discard a $C2$, a $D2$ or a $D3$ first. But if we first discard any of  $D1$, $D5$ and $D9$,  then we still have a chance to complete $T$ within two tile changes.

Suppose we discard $D5$ in the first tile change and replace it with an available tile $t=(c,n)$. If $t$ happens to be $D9$, then we have a complete 14-tile. The chance of replacing $D5$ with $D9$ in the first tile change is $1/4$ as there are, according to $\omega$, only 4 tiles are available.  If $t$ is not $D9$ (say it is $D2$), we cannot make a pair using it, as  there is no more $D2$ left. To complete $H[12/t]$ in one tile change, we have to discard the new obtained tile $t$ and change it to $D9$. The chance of this is equal to 
\begin{align*}
&\mbox{({\it the chance of changing $D5$ to $t$}) $\times$ ({\it the chance of changing $t$ to $D9$})}, 
\end{align*}
which is  $1/4 \times 1/3$ 
(note that, after changing $D5$ to $t$, we have only 3 tiles available). Therefore, the chance of completing $H$ within two tile changes if we first discard $D5$ is $1/4+3\times (1/4\times 1/3)=1/2$. 

Similarly, if we discard $D9$ first, then the chance of completing $H$ in one tile change is also $1/4$. Suppose $D9$ is replaced with a $D2$ or another $D9$ in the first tile change. Then, in order to complete the revised 14-tile, we have to discard the new given tile and replace it with $D5$. The success chance is also $1/4 \times 1/3$ for each of $D2$ and $D9$. But, if we replace $D9$ with $D4$, then, to complete $H[D9/D4]$, we may either discard $D4$ and replace it with $D5$ or discard $D1$ and replace it with either $D2$ or $D5$. Apparently, in the second tile change, it is more profitable to discard $D1$ instead of $D4$. The success chance is $1/4 \times \max(1/3,2/3)=1/4\times 2/3$. Thus the chance of completing $H$ within two tile changes if we first discard $D9$ is $1/4 + 2\times (1/4\times 1/3) + 1/4\times 2/3=7/12>1/2$. 

Lastly, suppose we discard $D1$ first. Then if we get a $D2$, then we cannot complete $H$ by one more tile change; if we get a $D4$, then need replace $D5$ with $D9$ or replace $D9$ with either $D2$ or $D5$; if we get a $D5$ ($D9$), then we need replace $D9$ ($D5$) with $D4$. The success chance is $0 + 1/4 \times \max(1/3,2/3) + 1/4\times 1/3 + 1/4\times 1/3 = 1/3$.

In summary, discarding $D9$ is more profitable than discarding any other tile.
\end{example}

From the above example, we can see that the value of a tile depends on how many tile changes we could have before the game is finished. For any 14-tile $T$, the \emph{step 0 value} of $T$ is defined as
\begin{align}
\val_0(T) &= 
    \begin{cases}
      1, & \mbox{if $T$ is complete}\\
      0, & \text{otherwise}
    \end{cases}
\end{align}

\begin{definition}
Suppose $T$ is an incomplete 14-tile and  $\omega$ is the current knowledge base of the agent. 
For any $0\leq i\leq 13$ and any $k>0$, we define the
\emph{step $k$ value} of the $i$-th tile of $T$ w.r.t. $\omega$, written  $\val_k(T,\omega,i)$, as the chance of completing $T$ within $k$ tile changes if $T[i]$ is discarded first. 
\end{definition}
For the special case when $k=1$, $\val_1(T,\omega,i)$ is indeed the success chance of obtaining a complete 14-tile by replacing $T[i]$ with an available tile.


Using the above notion, we further introduce a notion that measures the chance of completing an incomplete 14-tile $T$ within $k$ tile changes.
\begin{definition}
Suppose $T$ is an incomplete 14-tile and  $\omega$ is the current knowledge base of the agent. For any $k>0$, the \emph{step $k$ value} of $T$ w.r.t. $\omega$ is defined as
\begin{align}
\val_k(T,\omega) \equiv  \max_{0\leq i\leq 13} \val_k(T,\omega,i).
\end{align}
In case $T$ is complete, we define $\val_k(T,\omega)=1$.
\end{definition}

Clearly, $\val_k(T,\omega)=0$ if $\dfncy(T)>k$ or $0<\dfncy(T)\leq k$ but $\norm{\omega}=0$. 

We next give a recursive method for computing  $\val_k(T,\omega,i)$. First of all, we note that $\val_1(T,\omega,i)$ is the weighted sum over all available tiles $(c,n)$ of the chance of obtaining a complete 14-tile  by changing $T[i]$ as $(c,n)$, where the chance of selecting $(c,n)$ to replace $T[i]$ is $\omega(c,n)/\norm{\omega}$. Thus we have 

\begin{lemma}
For any incomplete 14-tile $T$, any knowledge base $\omega$ with $\norm{\omega}>0$, and any $0\leq i\leq 13$, we have 
\begin{align}
\label{eq:val1}
\val_1(T,\omega,i) &= \sum_{(c,n)\;:\;\omega(c,n)>0} \frac{\omega(c,n)}{\norm{\omega}} \times \val_0 \big(T[i/(c,n)] \big) .
\end{align} 
In case $\norm{\omega}=0$, we have $\val_1(T,\omega)=0$ for any $i$.
\end{lemma}
Comparing equations \eqref{eq:delta} and \eqref{eq:val1}, it is easy to see that 
\begin{align}\label{eq:delta=val1}
\val_1(T,\omega,i) = {\delta_{T,\omega}(i)}/{\norm{\omega}}
\quad\quad (0\leq i < 14).
\end{align}

If $T$ is not complete, the chance of completing $T$ within $k$ tile changes is the weighted sum of the chances of completing $T[i/(c,n)]$ within $k-1$ tile changes over $(c,n)$ with $\omega(c,n)>0$. Thus we have the following characterisation:
\begin{lemma}
For any incomplete 14-tile $T$, any knowledge base $\omega$ with $\norm{\omega}>0$, and any $0\leq i\leq 13$, we have 
\begin{align} 
   \val_k(T, \omega, i) &= \sum_{(c,n)\ :\ \omega(c,n)>0} \frac{\omega(c,n)}{\norm{\omega}} \times \val_{k-1}\big(T[i/(c,n)],\omega\!-\!(c,n)\big)
\end{align}
where $\omega\!-\!(c,n)$ denotes the knowledge base obtained by decreasing $\omega(c,n)$ by 1, as this tile $(c,n)$ becomes unavailable after replacing $T[i]$ with it. 
\end{lemma}

The above two lemmas give a recursive method for selecting the tile to discard:
\paragraph{The best tile to discard}
Suppose $T$ is an incomplete 14-tile and $\omega$ is a knowledge base with $\norm{\omega}>0$. Then $T[i] = \discard_k(T,\omega)$ is the tile that has the best chance for completing $T$ within $k$ tile changes if we discard $T[i]$ first, where
\begin{align}\label{eq:discard-k}
\discard_k(T,\omega) = \argmax_{T[i]\ :\ 0\leq i\leq 13} \val_{k}(T,\omega,i)
\end{align}
By Eq.~\eqref{eq:delta=val1},  $\discard_1(T,\omega)$ is the same as the tile $\discard(T,\omega)$ defined in Eq.~\eqref{eq:discard}.

\section{Conclusions and Future Work}
In this paper, we have initiated a mathematical and AI study of the Mahjong game. The definition of the deficiency number as well as the notions of knowledge base and step $k$ value will play important role in devising efficient computer programs for playing Mahjong. This is the topic of an ongoing research. 

Despite of its extremely popularity, there are very few mathematical or AI research papers which are devoted to the study of the Mahjong game. To our best knowledge, \cite{Cheng18} is the first serious attempt to study the Mahjong game using mathematical (mainly elementary combinatorial theory) techniques. In their paper, Cheng, Li and Li studied a special combinatorial problem in Mahjong game, viz., the so-called \emph{$k$-gate problem}. A pure 13-tile $T$ is called a \emph{nine-gate} if $T$ becomes complete if we add any tile in the same colour. In general, for $1\leq k\leq 9$, $T$ is called an $k$-gate if there are $k$ tiles with different values such that each of these $k$ tiles can complete $T$ but no any other tile can do the same job. It is easy to see that the $k$-gate problem can also be described and solved in the formalism developed in our work. To find all  $k$-gates, one need only decide, for each pure 13-tile $T$, and any tile $t=i$ ($1\leq i\leq 9$) with the same colour, whether there are exactly $k$ tiles such that $T$ plus tile $i$ is complete.   

There are at least three directions to extend the above work. First, we may include more tiles in $\mahjong_0$, e.g., the \emph{winds}, the \emph{dragons}, and the \emph{bonus} (see Figure~\ref{fig:mjtiles}). Second, we may expand and/or restrict the set of legal 14-tiles. For examples, we may allow any seven pairs as a complete 14-tile, or we may require that any complete 14-tile has at most two colours. Third, different complete 14-tiles may have different scores. For example, a pure complete 14-tile may worth much more than a hybrid one. Future work will address these issues and adapt our methods according to different score systems. 

\appendix
\section{Proof of Proposition~\ref{prop0}}
\label{sec:appendix1}
\begin{proof}

If $V$ contains four disjoint melds, then $V$ can be completed within two tile changes by constructing a new pair. 

Suppose $V$ contains three, but no more, disjoint melds. Let $W=(uvwxy)$ be the  (not sorted) subsequence  containing the rest of $V$. Since $W$ does not contain a meld, it has at most two pairs. Suppose $p,q,r$ are any three pairwise different tiles in $W$. We assert that there is one of $p,q,r$ which can be completed into a meld within two tile changes.  This is because, if, say, $p$ cannot be completed into a pong, then $V\setminus W$ must contain at least two identical tiles $p$, i.e., $(pp)\subseteq V\setminus W$. That is, if $(pp)\subseteq W$, then $V$ must contain the kong $(pppp)$; if $(pp)$ is not contained in $W$, then $V$ must contain the pong $(ppp)$. Similarly, suppose $(pab)$ is a chow that is not completable. Then $(ab)$ is not contained in $W$ and, hence, either $a$ or $b$ is not in $W$. If $a\in W$, then $V\setminus W$ must contain $(bbbb)$; if $b\in W$, then $V\setminus W$ must contain $(aaaa)$; if neither $a$ nor $b$ is in $W$, then $V$ must contain either $(aaaa)$ or $(bbbb)$. Thus, if none of $p,q,r$ is completable, then $T$ must contain three pongs and at least one kong $(zzzz)$ other than $(pppp),(qqqq),(rrrr)$. The existence of $(zzzz)$ is due to that $(pqr)$ is not a meld. This, however, contradicts the assumption that $V$ contains no more than three disjoint melds. 

Suppose $W$ contains two pairs say $u=v$ and $x=y$. If $w$ can be completed into a meld, then either $x$ or $u$ is not involved in the meld. Thus completing $w$ into a meld and using either $(xy)$ or $(uv)$ as the eye will complete $V$ within two tile changes. If $w$ cannot be completed into a meld, then $V$ contains $(www)$ and, because no chow $(wab)$ is completable, there exists a tile $w'\in \{a,b\}$ and $w'\not=u,x$ such that $(w'w'w'w')\subset V$. If $V$ contains both $(xxx)$ and $(uuu)$, then, together with $(www)$ and $(w'w'w')$, we shall have four pongs in  $V$, a contradiction. Thus either $(xxx)$ or $(uuu)$ is not contained in $V$. This implies that  in this subcase $V$ can be completed by completing $(xy)$ into $(xxx)$ or completing $(uv)$ into $(uuu)$ using one tile change.

Suppose $W$ contains exactly one pair, say, $u=v$ and $u,x,y,w$ are pairwise different. If none of $u,x,y,w$ can be completed into a pong, then by previous analysis $V$ should contain $(uuuu)$, $(xxx), (yyy)$ and $(www)$, which contradicts the assumption that $W$ does not contain four disjoint melds. If there is one of $x,y,w$ that is completable into a pong or a chow does not involve $u=v$, then $V$ can be completed within two tile changes. Suppose this is not the case. Then $V$ should contain $(xxx),(yyy),(www)$ but not $(uuu)$, and one of $x,y,w$, say $x$, can be completed into a chow $(xuz)$. Then $V$ can be completed within two tile changes by constructing the chow $(xuz)$ and use the pair $(uv)$ as eye.

If $W$ contains no pair, then $u,v,w,x,y$ are pairwise different. Suppose $u<v<w<x<y$. After routine check, we can see that there exist two tiles $z_1,z_2$ such that $z_1<v<w<x<z_2$ and $(uvwz_1)$ and $(wxyz_2)$ contain two disjoint chows.\footnote{This is because $v-u>2$ and $y-x>2$ cannot happen simultaneously. For example, if $W=(12459)$, then we may select $z_1=3$ and $z_2=6$, and we have chows $(123)$ and $(456)$.} If neither $z_1$ nor $z_2$ is available, then $V\setminus W$ contains two kongs $(z_1z_1z_1z_1)$ and $(z_2z_2z_2z_2)$. This is impossible as $V\setminus W$ consists of three melds and $z_2>z_1+2$. Suppose without loss of generality $z_1$ is available and $(vwz_1)$ is a chow. If $V$ does not contain all of $(uuuu),(xxxx)$ and $(yyyy)$, then we may construct a pair, thus complete $V$, by one more tile change. In case $V$ does contain all of $(uuuu),(xxxx)$ and $(yyyy)$, then the other two tiles of $V$ are $v$ and $w$. Together with three pongs $(uuu),(xxx),(yyy)$, we may complete $V$ by forming the chow in $(wxyz_2)$ (by replacing $u$ with $z_2$) and form a pair $(vv)$ by replacing with $v$ the tile in $(wxy)$ that is not in the chow involving $z_2$. 

In the last, we consider the case when  $V$ contains two, and no more, disjoint melds. Let $W$ be the subsequence containing the rest eight tiles of $V$. Since it has no meld, $W$ contains two to four pairs. For any two pairs $(xx)$ and $(yy)$ contained in $W$, if neither can be completed into a pong, then $V\setminus W$ contains two identical $x$ and two identical $y$. Since $V\setminus W$ consists of two melds, the two melds must be chows with form $(xya)$ and $(xyb)$, where $a=b$ is possible. This implies, however, that $V$ contains three disjoint melds $(xxx)$, $(yyy)$ and $(xya)$, a contradiction. Therefore, if $W$ contains $2\leq k\leq 4$ different pairs, then $k-1$ of these pairs can be completed into pongs. If $k\geq 3$, then $V$ can be completed within two tile changes. If $k=2$, then let $(xx)$ and $(yy)$ be the two pairs and $u,v,w,z$ be the rest four tiles (not sorted) in $W$. By Lemma~\ref{fact2}, there is a pchow contained in $(uvwz)$. Let $(uv)$ be this pchow. If $(uv)$ is completable, then together with a completable pair $(xx)$ or $(yy)$, we can complete $V$ within two tile changes. Suppose this is not the case and $(uv)$ is not completable. Then, for any $a$ such that $(uva)$ (after sorting) is a chow, $V\setminus W$ contains the kong $(aaaa)$. Because $V\setminus W$ consists of two melds, this is possible only if $a$ is between $u,v$ and $V\setminus W = (aaa)(uva)$. Note that $a$ cannot be either $x$ or $y$. Both pairs are completable as $V\setminus W = (aaa)(uva)$ does not contain $x$ or $y$. Therefore, we can complete $V$ by complete both $(xx)$ and $(yy)$, obtaining four melds $(xxx),(yyy),(uva),(uva)$ and a pair $(aa)$.
\end{proof}

\section{Proof of Proposition~\ref{prop:normal}}
\label{sec:appendix2}
\begin{proof}
By Lemma~\ref{lem:saturated}, if $\pi$ is not saturated, there exists a saturated p-decomposition with a smaller cost. In the following, we assume $\pi$ is saturated.

Suppose $\pi(5)\not=\varnothing$. Without loss of generality, let $\pi(4)=\varnothing$. Clearly, $\pi(0)$ contains at least 3 tiles. Since $\pi$ is saturated, $\pi(0)$ does not contain any meld, nor any pmeld or single tile that is completable. This implies that, for any $x$ in $\pi(0)$, $T$ must contain the pong $(xxx)$. Moreover, for any chow $(xyz)$ ($x,y,z$ may be not ordered) containing $x$, $T$ must contain either $(yyyy)$ or $(zzzz)$. In particular, for any $x\in \pi(0)$, $T$ must 
\begin{itemize}
    \item contain $(xxx)$, and 
    \item contain either $(x^{--}x^{--}x^{--}x^{--})$ or $(x^{-}x^{-}x^{-}x^{-})$ if $x\geq 3$, and
    \item contain either $(x^{-}x^{-}x^{-}x^{-})$ or $(x^{+}x^{+}x^{+}x^{+})$ if $2\leq x\leq 8$, and
    \item contain either  $(x^{++}x^{++}x^{++}x^{++})$ or $(x^{+}x^{+}x^{+}x^{+})$ if $x\leq 7$. 
\end{itemize}

Let $u,v,w$ be three tiles contained in $\pi(0)$. By our assumption, $(u,v,w)$ is neither a chow nor a pong.   

\paragraph{Case 1} Suppose $u,v,w$ are in different colours. By the previous analysis,  $T$ should contain $(uuu),(vvv),(www)$ and at least three kongs in different colours. This is impossible as $T$ is a 14-tile. 

\paragraph{Case 2} Suppose $u<v$ are in the same colour but $w$ is in a different colour. As above, $T$ should contain $(uuu),(vvv),(www)$ and at least two kongs in different colours. This is also impossible as $T$ is a 14-tile.

\paragraph{Case 3} Suppose $u<v<w$ are in the same colour.  As above, $T$ should contain $(uuu),(vvv)$, $(www)$. Moreover, if $w<8$ ($u>2$), then $T$ should contain $(xxxx)$ for some $x>w$ ($x<u$); if $w>v+1$ ($v>u+1$), then $T$ should contain $(yyyy)$ for some $v<y<w$ ($u<y<v$).  Because $(uvw)$ is not a chow, we have $w>u+2$, and thus have either $w>v+1$ or $u>u+1$. To not violate the constraint that $T$ contains 14 tiles, we can have at most one kong $(xxxx)$ with $x\not\in \{u,v,w\}$. This is only possible when $u\leq 2$, $w\geq 8$ and either $v=u+1$ or $v=w-1$, which implies $w-v\geq 5$. But in this case, $T$ should also contain either $(w^-w^-w^-w^-)$ or $(w^{--}w^{--}w^{--}w^{--})$ and either $(v^+v^+v^+v^+)$ or $(v^{++}v^{++}v^{++}v^{++})$. A contradiction.


\paragraph{Case 4} Suppose $u=v<w$ are in the same colour. In this case, $T$ should contain $(uuuu)$ and $(www)$. 

When $u=v=1$ and $w=2$, if $T$ does not contain $(2222)$, then it should contain $(3333)$; if $T$ contains $(2222)$, then it should also contain either $(3333)$ or $(4444)$. In the first subcase, we should have that $(11223333)$ is contained in $\bigcup_{i=1}^5 \pi(i)$ and $(11112223333)$ is contained in $T$. Assume $x,y,z$ are the rest tiles contained in $T$. Since $x,y,z$ cannot be $1,2,3$, we have a new completable decomposition $\hat{\pi}=(111)(222)(333)(13)(x)$, which has cost smaller than $\pi$. In the second subcase, without loss of generality, we assume that $T$ contains $(1111),(2222)$ and $(4444)$ and two additional tiles $x,y$. Let  
$\hat{\pi}=(111)(222)(444)(12)(x)$. Then $\hat{\pi}$ is a completable p-decomposition which has a smaller cost than $\pi$.

The case when $u=v=8$ and $w=9$ is dually analogous. 
    
Now suppose $1<u=v<8$ and $u<w$. If $w=u+1$, then, as $(uw)$ is not completable, $\bigcup_{i=1}^5 \pi(i)$ must contain $(u^-u^-u^-u^-)$ and $(w^+w^+w^+w^+)$. Recall that $T$ already contains $(uuuu)$ and $(www)$. This contradicts the fact that $T$ is a 14-tile. 
Similarly, if $w=u+2$, then $\bigcup_{i=1}^5 \pi(i)$ contains $(u^+u^+u^+u^+)$. Let $x,y,z$ be the rest three tiles in $T$ and assume without loss of generality $x,y,z$ are not $w$. This shows that $T$ has a better p-decomposition $$\hat{\pi}=(uu^+w)(uu^+w)(uu^+w)(uu^+)(x).$$
If $w>u+2$, then we shall have at least two tiles $x,y$ not in $\{u,w\}$ s.t. $T$ contains $(xxxx)$ and $(yyyy)$. This, again, violates the assumption that $T$ is a 14-tile.

\paragraph{Case 5} The case when  $u<v=w$ are in the same colour is analogous to the previous one.

Suppose $\pi(5)=\varnothing$ and $\pi(0)$ contains 4 or more tiles. Since $\pi$ is saturated, there are no pairs in $\pi(0)$. If there is a tile $t$ in $\pi(0)$ such that $T$ contains 3 or less identical tiles $t$, then we may replace $\pi(5)$ with the single tile $(t)$ and obtain a new completable p-decomposition with a smaller cost. If this is not possible for any $x\in \pi(0)$, then, for each $x\in \pi(0)$, $T$ should contain the pong $(xxx)$. As there are four or more tiles in $\pi(0)$, we have a new completable p-decomposition $\pi^*$ consists of four pongs and, possibly, an empty $\pi^*(5)$, which has a smaller cost. 
\end{proof}

As a direct corollary, we know 
\begin{coro}
\label{coro1}
Let $T$ be a 14-tile and $\pi$ a completable p-decomposition of $T$. If $\pi(i)=\varnothing$ for some $1\leq i\leq 4$, then there exists another completable p-decomposition $\hat{\pi}$ of $T$ such that $\hat{\pi}(i)\not=\varnothing$ for any $1\leq i\leq 4$ and $\cost_T(\hat{\pi}) < \cost_T(\pi)$.
\end{coro}

\bibliographystyle{plain}
\bibliography{mahjong}

\end{document}